\documentclass{ifacconf}

\usepackage[T1]{fontenc}

\usepackage{lmodern}
\usepackage{graphicx}      
\usepackage{natbib}        
\usepackage{algorithm}
\usepackage{algpseudocode}
\usepackage{amsmath,amssymb,amsfonts}
\begin{document}
\begin{frontmatter}

\title{Constrained Policy Optimization via Sampling-Based Weight-Space Projection\thanksref{footnoteinfo}} 

\thanks[footnoteinfo]{The work is sponsored by the Department of the Navy, Office of Naval Research ONR N00014--24--2099.  The content of the information does not necessarily reflect the position or the policy of the Government, and no official endorsement should be inferred.}

\author[UCB]{Shengfan Cao} 
\author[UCB]{Francesco Borrelli} 
\author[Second]{Eunhyek Joa}

\address[UCB]{University of California at Berkeley, 
   Berkeley, CA 94720 USA (e-mail: \{shengfan\_cao, fborrelli\}@berkeley.edu).}
\address[Second]{Department of Mechanical Engineering,
   Seoul National University, Seoul, Korea, (e-mail: eunhyekj@snu.ac.kr)}

\begin{abstract}                
    Safety-critical learning requires policies that improve performance without leaving the safe operating regime. We study constrained policy learning where model parameters must satisfy rollout-based safety constraints that can be evaluated but not differentiated analytically. We propose SCPO, a sampling-based weight-space projection method that enforces safety directly in parameter space without requiring gradient access to the constraint functions. SCPO constructs a local safe region by combining rollout-based safety evaluations with smoothness bounds relating parameter perturbations to changes in safety metrics, and projects each gradient update via a convex QCQP. We establish a safe-by-induction guarantee: starting from any safe initialization, all intermediate policies remain safe given feasible projections. In constrained control settings with a stabilizing backup policy, SCPO further ensures closed-loop stability while enabling safe adaptation beyond the conservative backup. Experiments on constrained regression with harmful supervision and double-integrator imitation with a malicious expert show that SCPO rejects unsafe updates, maintains feasibility throughout training, and achieves meaningful objective improvement.
\end{abstract}

\end{frontmatter}

\newcommand{\pisafe}{\pi_{\text{safe}}}
\newenvironment{theorem}{\begin{thm}}{\end{thm}}
\newenvironment{proof}{\begin{pf}}{\end{pf}}
\newenvironment{lemma}{\begin{lem}}{\end{lem}}
\newenvironment{remark}{\begin{rem}}{\end{rem}}
\newenvironment{definition}{\begin{defn}}{\end{defn}}
\newenvironment{assumption}{\begin{assum}}{\end{assum}}
\newenvironment{proposition}{\begin{prop}}{\end{prop}}

\newif\ifextended
\extendedtrue

\section{Introduction}

Safety is a fundamental requirement for deploying learning-based controllers in domains such as autonomous driving and robotics. 
Deep neural network (NN) policies can achieve high performance, but formal guarantees for constraint satisfaction remain challenging. 
In practice, real-world systems often use a backup controller that outputs safe but conservative actions. 
While such backups can prevent catastrophic failures~\citep{wagenerSafeReinforcementLearning2021,zhangSafetyCorrectionBaseline2022}, they are typically used as hard switches and are not integrated into the learning process.

A key observation in this work is that a backup controller need not only be a fail-safe mechanism; it can serve as a feasible initialization for safe adaptation.
Inspired by low-rank adaptation (LoRA)~\citep{huLoRALowRankAdaptation2021}, we treat a certified backup controller as a frozen base policy and augment it with a NN policy that outputs zero at initialization. 
The combined policy therefore coincides with the backup controller at the start of training, and subsequent learning can be viewed as safe adaptation in the neighborhood of a known safe policy.

Existing work on safe adaptation can be roughly grouped into three categories. 
\emph{Output-level safety} uses safety filters, control barrier functions, or MPC-based shields to modify the policy's action post-hoc; these methods offer strong guarantees but the NN is typically unaware of the filter. 
\emph{Data-level safety} removes unsafe trajectories or biases sampling toward safe regions, but is less applicable to on-policy learning with limited data. 
\emph{Training-time safety} constrains the optimization process through trust regions, constrained policy gradients, or gradient projections, often in CMDP settings. 

We propose SCPO, a data-driven, projection-based method for constrained learning that provides stability and recursive feasibility certification. 
The key idea is to locally estimate a safe set in weight space from small perturbations in weights and the associated changes in safety metrics. 
Using sampled rollouts and local smoothness assumptions, we construct a tractable quadratic inner approximation and project each raw gradient update onto it, preserving safety on all considered rollouts at every training step.

Our main contributions are: (i) a sampling-based constrained policy optimization that performs weight-space safety projection; (ii) a safe-by-induction learning mechanism for monotonic improvement and persistent constraint satisfaction; (iii) a control-theoretic analysis of SCPO for linear dynamical systems; and (iv) two simple experiments demonstrating robustness against harmful supervision and safe adaptation.

\section{Related Work}
Safe learning has been extensively studied in CMDPs and robotics; see~\cite{garciaComprehensiveSurveySafe2015} and~\cite{guReviewSafeReinforcement2024} for a comprehensive overview. 

\subsubsection{Primal-dual and trust-region methods}
A large body of work formulates safety as expected accumulated cost and applies Lagrangian or primal-dual optimization to CMDPs. 
CRPO~\citep{xuCRPONewApproach2021} alternates between reward maximization and cost minimization, while NPG-PD~\citep{dingNaturalPolicyGradient2020} performs natural policy gradient ascent on primal variables and sub-gradient descent on dual variables. 
TRPO~\citep{schulmanTrustRegionPolicy2017} and PPO~\citep{schulmanProximalPolicyOptimization2017} restrict each update by constraining the expected KL-divergence between successive policies. 
These approaches focus on expected accumulated cost or expected return, and typically do not ensure trajectory-level safety during training.

\subsubsection{Projected updates} 
Projected-update methods enforce safety by projecting a potentially unsafe update onto a safe set. 
CPO~\citep{achiamConstrainedPolicyOptimization2017} optimizes a surrogate reward subject to linearized cost constraints and a KL trust-region in policy space, while Lyapunov-based methods~\citep{chowLyapunovbasedApproachSafe2018} represent safety as linear constraints within dynamic programming algorithms. 
In LLMs, SafeLoRA~\citep{hsuSafeLoRASilver2025} and SaLoRA~\citep{liSaLoRASafetyAlignmentPreserved2025} project LoRA updates to be approximately aligned with a safety-aligned subspace. 
In contrast, we approximate a feasible set in weight space induced purely by a backup controller and safety metrics, without explicit constraint costs or an aligned reference model.

\subsubsection{Projected action/policy}
Projection in the action or policy space is another popular approach. 
For example, MPC-based projection can be imposed on the action output to enforce recursive feasibility~\citep{grosSafeReinforcementLearning2020}. 
Other approaches learn a mapping from an unconstrained high-dimensional space to the feasible set~\citep{liuRobotReinforcementLearning2022}, or use differentiable projection layers for specific QP problems~\citep{minHardNetHardConstrainedNeural2025}. 
These methods guarantee safety via the projection layer or safety filter, but the underlying policy parameters themselves are not certified and the constraint set must be explicitly known.

\section{Preliminaries}

\subsection{System and Optimal Control Task}
Consider a nonlinear, time-invariant, discrete-time, deterministic system described by:
\begin{equation}\label{eq:sys_dynamics}
   \begin{aligned}
      x_{k+1} = f(x_k, u_k),
   \end{aligned}
\end{equation}
where $x_k$ is the state and $u_k$ is the control input. 
The system is subject to constraints:
\begin{equation}
   x_k \in \mathcal{X},~u_k \in \mathcal{U},~\forall k \geq 0.
\end{equation}
We consider the infinite-horizon constrained optimal control problem:
\begin{equation}
\label{eq:control_task_formulation}
\begin{aligned}
    \min_{\mathbf{u}_{0:\infty}} ~~ & \sum_{k=0}^{\infty} c(x_k, u_k) \\
    \textnormal{s.t.} ~~
    & x_{k+1} = f(x_k, u_k), ~~ x_0 \in \mathcal{X}_0, \\
    & x_k \in \mathcal{X}, ~~ u_k \in \mathcal{U}, ~~ \forall k \in \mathbb{N},
\end{aligned}
\end{equation}
where $\mathcal{X}_0 \subseteq \mathcal{X}$ is the set of initial states, $c : \mathbb{R}^{n_x} \times \mathbb{R}^{n_u} \to \mathbb{R}_{\ge 0}$ is a continuous, non-negative stage cost.

\subsection{Baseline Controller}

\begin{assumption}[Baseline controller]\label{asm:backup-controller}
   Assume we are given a non-parametric feedback controller $\pisafe: \mathcal{X} \to \mathcal{U}$ and a set $\Omega \subseteq \mathcal{X}$ with $x_e \in \Omega$, such that
   \begin{equation}
      x_e = f(x_e, \pisafe(x_e)).
   \end{equation}
   Furthermore, suppose there exists a continuously differentiable Lyapunov function $V: \Omega \to \mathbb{R}_{\ge 0}$ and class-$\mathcal{K}$ functions $\underline{\alpha}, \overline{\alpha}$, and $\alpha$ such that, for all $ x \in \Omega$, 
   \begin{align}
      \pisafe(x) &\in \mathcal{U}, \\
      f(x, \pisafe(x)) &\in \Omega, \\
      \underline{\alpha}(\Vert x - x_e \Vert) \leq V(x) &\leq \overline{\alpha}(\Vert x - x_e \Vert), \\
      V(f(x, \pisafe(x))) - V(x) &\leq -\alpha(\Vert x - x_e \Vert).
   \end{align}
\end{assumption}

\subsection{Residual Controller}
The safe controller $\pisafe$ need not be high-performing with respect to~\eqref{eq:control_task_formulation}. 
In this work, we treat $\pisafe$ as the initial policy, and add a residual policy in parallel with $\pisafe$ to improve the performance indicated by~\eqref{eq:control_task_formulation} while ensuring the combined controller remains stabilizing in $\Omega$. 
Specifically, define the target set
$\mathcal{X}_f = \mathcal{B}_\epsilon(x_e) \subseteq \Omega$
as a smaller neighborhood around $x_e$, where $\epsilon > 0$.
Let our feedback controller be: 
\begin{equation}\label{eq:controller}
   \pi_\theta(x) = \begin{cases}    
      \pi_{\text{safe}}(x) + \phi_\theta(x),& x \notin \mathcal{X}_f, \\
      \pi_{\text{safe}}(x),& x \in \mathcal{X}_f,
   \end{cases}
\end{equation}
where $\phi_\theta: \mathcal{X} \to \mathbb{R}^{n_u}$ is a neural network parameterized by $\theta$. 
Note that $x_e$ is also an equilibrium for the closed-loop system $x_{k+1} = f(x_k, \pi_\theta(x_k))$ because $\pi_\theta = \pisafe$ near $x_e$. 

Inspired by LoRA~\citep{huLoRALowRankAdaptation2021}, we may initialize the weights $\theta_0$ such that the final linear layers are initialized to zeros, as illustrated in Fig.~\ref{fig:architecture}, which guarantees $\theta_0$ satisfies $\phi_{\theta_0}(x) \equiv \mathbf{0}$. Therefore $\pi_{\theta_0}(x) = \pi_{\text{safe}}(x), ~\forall x \in \mathcal{X}$. 
\begin{figure}[htbp]
    \centering
    \includegraphics[width=0.65\linewidth]{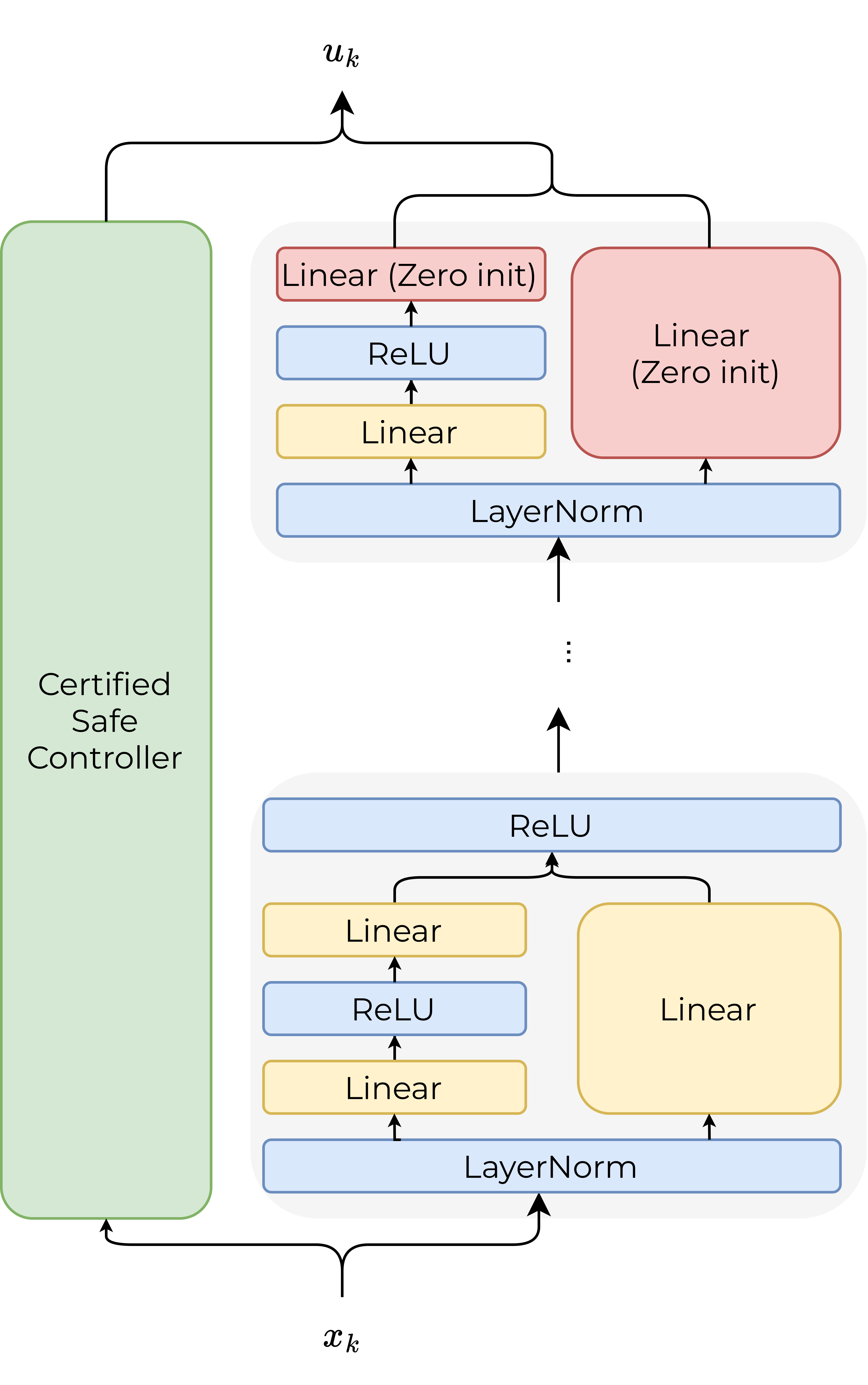}
    \caption{\textbf{Architecture of $\pi_\theta$} in this work. $\pi_{\text{safe}}$ (left) is a certified safe controller, while the neural network $\phi_\theta$ (right) acts as a residual policy. The final linear layers (red) are zero-initialized so the initial policy matches $\pi_{\text{safe}}$. The gray block is repeated six times.}
    \label{fig:architecture}
\end{figure}

\begin{remark}
Although $\pi_{\text{safe}}$ could in principle be high-capacity, in many safety-critical settings a simple stabilizing controller is easier to certify and makes the initialization $\pi_\theta=\pi_{\text{safe}}$ (with $\phi_\theta=\mathbf{0}$) trivially safe. 
By contrast, encoding such conservative behavior directly inside a deep network is harder, so we delegate performance improvements to the expressive residual $\phi_\theta$, which is subsequently projected to preserve safety.
\end{remark}

\subsection{Stability Under Lyapunov-Margin Preservation}
\begin{proposition}\label{prop:stability-under-lyap-margin-preservation}
   Suppose Assumption~\ref{asm:backup-controller} holds. 
   Let $\pi_\theta$ be a parametric feedback controller of the form~\eqref{eq:controller} that satisfies
   \begin{equation}
      x_e = f(x_e, \pi_\theta(x_e)).
   \end{equation}
   If there exists $\gamma \in [0, 1)$ such that, for all $x \in \Omega$, 
   \begin{align}
      \pi_\theta(x) &\in \mathcal{U}, \label{eq:input-feasibility}\\
      f(x, \pi_\theta(x)) &\in \Omega, \label{eq:state-feasibility}\\
      V(f(x, \pi_\theta(x))) - V(x) 
      &\leq 
      -(1 - \gamma) \alpha(\Vert x - x_e \Vert), \label{eq:lyap-margin-preservation}
   \end{align}
   then $\Omega$ is forward invariant under $\pi_\theta$, and $x_e$ is locally asymptotically stable for the closed-loop system $x_{k+1} = f(x_k, \pi_\theta(x_k))$. 
\end{proposition}

Proposition~\ref{prop:stability-under-lyap-margin-preservation} follows from standard discrete-time Lyapunov stability arguments~(e.g.,~\citealt{khalil_nonlinear_2002}). 
We interpret the conditions in~\eqref{eq:input-feasibility}--\eqref{eq:lyap-margin-preservation} as constraints on $\theta$, which we enforce during training in our proposed approach. 

\section{Problem Formulation}\label{sec:projection}

Our goal is to construct a sequence of policy weights $\theta_0,\theta_1,\ldots$ such that the corresponding controllers $\pi_{\theta_i}$, defined in~\eqref{eq:controller}, improve the task objective in~\eqref{eq:control_task_formulation} while remaining stabilizing on the certified region $\Omega$ for all iterations
$i$.

\subsection{Constrained Adaptation}

At each iteration, we assume that the current weights $\theta_t \in
\mathbb{R}^d$ are feasible, i.e.,
\begin{equation}
   g(\theta_t) \leq \mathbf{0},
\end{equation}
where $g:\mathbb{R}^d \to \mathbb{R}^{n_g}$ denotes a vector of safety
constraint functions. Our goal is to compute an update
$\Delta\theta = \theta_{t+1}-\theta_t$ that improves a differentiable
task loss while preserving feasibility:
\begin{equation}
    \label{eq:residual_learning_problem}
    \begin{aligned}
        & \min_{\Delta \theta} \quad \mathcal{L}(\theta_t + \Delta \theta) \\
        & \text{s.t.} \quad g(\theta_t + \Delta \theta) \le \mathbf{0}.
    \end{aligned}
\end{equation}
The inequality is understood componentwise.

In general, the safety map $g$ may be difficult to express analytically
for neural policies. We therefore assume that $g(\theta)$ can be evaluated
pointwise, but do not require closed-form access to its Jacobian. The
projection step introduced below is constructed from sampled evaluations
of $g$ around the current iterate.

To keep the approximation local, we restrict each update to a trust
region:
\begin{equation}
   \Vert \Delta \theta \Vert_2 \leq r.
\end{equation}
We also assume that the safety constraints are locally smooth within this
region.

\begin{assumption}\label{ass:l_smooth}
For each component $g_i$, there exists $\ell_i \geq 0$ such that, for all
$\Delta\theta$ satisfying $\Vert\Delta\theta\Vert_2 \leq r$,
\[
\left|
g_i(\theta_t+\Delta\theta)
-
g_i(\theta_t)
-
\nabla_\theta g_i(\theta_t)^\top \Delta\theta
\right|
\le
\frac{\ell_i}{2}\Vert\Delta\theta\Vert_2^2 .
\]
\end{assumption}

\begin{remark}
The constants $\ell_i$ are difficult to compute exactly for neural policies. In practice, we initialize them using the maximum observed pairwise slope over sampled updates, which is only meaningful locally and motivates the trust-region constraint $\Vert\Delta\theta\Vert_2\le r$. Since the residual policy is initialized from the backup controller, this also keeps learning in a local fine-tuning regime where the curvature of $g$ is more stable empirically. If the quadratic bound is violated, we increase $\ell_i$ geometrically by a factor of $\kappa$ and recompute the update.
\end{remark}


\subsection{Local Quadratic Surrogate}

At iteration $t$, we construct a local surrogate of~\eqref{eq:residual_learning_problem} around the current feasible weights
$\theta_t$. For the task objective, we use a first-order Taylor expansion
of $\mathcal L$ together with a regularization term:
\begin{equation}\label{eq:surrogate-loss}
\hat{\mathcal{L}}(\theta_t + \Delta \theta)
=
\mathcal{L}(\theta_t)
+ \nabla_\theta \mathcal{L}(\theta_t)^\top \Delta \theta
+ \lambda \| \Delta \theta \|_2^2,
\end{equation}
where $\lambda>0$ controls the step size.

The unconstrained minimizer of~\eqref{eq:surrogate-loss} is
\begin{equation}\label{eq:unconstrained-gradient-step}
    \Delta \theta_{\mathrm{raw}}
    \triangleq
    -\eta \nabla_\theta \mathcal{L}(\theta_t),
    \qquad
    \eta = \frac{1}{2\lambda}.
\end{equation}
Equivalently, up to an additive constant,
\begin{equation}
    \hat{\mathcal{L}}(\theta_t + \Delta \theta)
    =
    \lambda
    \Vert \Delta \theta - \Delta \theta_{\mathrm{raw}}\Vert_2^2
    + \mathrm{const}.
\end{equation}

For the constraints, Assumption~\ref{ass:l_smooth} implies that, for each
$i$ and all $\Vert \Delta \theta \Vert_2 \le r$,
\begin{equation}
    g_i(\theta_t+\Delta\theta)
    \le
    g_i(\theta_t)
    +
    \nabla_\theta g_i(\theta_t)^\top \Delta\theta
    +
    \frac{\ell_i}{2}\|\Delta\theta\|_2^2 .
\end{equation}
Therefore, a conservative sufficient condition for
$g_i(\theta_t+\Delta\theta)\le 0$ is
\begin{equation}
    g_i(\theta_t)
    +
    \nabla_\theta g_i(\theta_t)^\top \Delta\theta
    +
    \frac{\ell_i}{2}\|\Delta\theta\|_2^2
    \le 0,
    \qquad i=1,\ldots,n_g.
\end{equation}

Let $J_g(\theta_t)\in\mathbb R^{n_g \times d}$ denote the Jacobian of $g$
at $\theta_t$, and let $L=[\ell_1,\ldots,\ell_{n_g}]^\top$. The local surrogate
projection problem is
\begin{equation}
    \label{eq:safety_projection_problem_linear}
    \begin{aligned}
        \min_{\Delta \theta} \quad
            & \| \Delta \theta - \Delta \theta_{\mathrm{raw}} \|_2^2 \\
        \mathrm{s.t.} \quad
            & g(\theta_t) + J_g(\theta_t)\Delta \theta
            + \frac{1}{2} \Vert \Delta \theta \Vert_2^2 L
            \le \mathbf{0}, \\
            & \|\Delta\theta\|_2 \le r .
    \end{aligned}
\end{equation}

\section{Sampling-Based Weight-Space Projection}
Directly solving~\eqref{eq:safety_projection_problem_linear} in the full parameter space is generally impractical for large neural policies because the dimension $d$ is large, and the Jacobian $J_g(\theta_t)$ is typically unavailable in closed form. We therefore construct a conservative sampling-based approximation using only pointwise evaluations of $g$ at a finite set of candidate updates.

\subsection{Sampling-Based Projection}

At iteration $t$, suppose we have a nonempty bank of candidate updates
\begin{equation}
   D =
   \begin{bmatrix}
      \Delta \theta^{(1)} & \cdots & \Delta \theta^{(m)}
   \end{bmatrix}
   \in \mathbb{R}^{d \times m},
\end{equation}
with $\Vert \Delta\theta^{(i)} \Vert_2 \le r$ for all $i$, together with their safety
evaluations
\begin{equation}
   G =
   \begin{bmatrix}
      g^{(1)} & \cdots & g^{(m)}
   \end{bmatrix}
   \in \mathbb{R}^{n_g \times m},
   \qquad
   g^{(i)} \triangleq g(\theta_t + \Delta \theta^{(i)}).
\end{equation}
In practice, $D$ and $G$ are obtained by caching recent model updates and
their corresponding safety evaluations during training. We assume that
the most recent unconstrained update is included as the last candidate,
i.e.,
\begin{equation}
   \Delta \theta_{\mathrm{raw}} = \Delta \theta^{(m)}, 
\end{equation}
which we scale down to ensure it lies in the trust region:
\begin{equation}
   \Vert \Delta\theta_{\mathrm{raw}} \Vert_2 \le r .
\end{equation}
As shown below, because the sampled projection is a Euclidean projection
onto a convex set containing the zero update, the projected update cannot
have larger norm than $\Delta\theta_{\mathrm{raw}}$. Hence the projected
update also remains in the trust region.

Instead of optimizing over all $\Delta\theta\in\mathbb R^d$, we restrict
the update to the span of the sampled candidates:
\begin{equation}
    \Delta\theta = D\xi,
    \qquad \xi\in\mathbb R^m.
\end{equation}
Let
\begin{equation}
   S = D^\top D,
   ~~
   \mathrm{diag}(S)
   =
   [\Vert \Delta\theta^{(1)} \Vert_2^2, \dots ,\Vert \Delta\theta^{(m)} \Vert_2^2]^\top .
\end{equation}

\begin{proposition}\label{prop:reformulation}
Consider the following sampling-based projection problem:
\begin{equation}\label{eq:sampled_projection}
   \begin{aligned}
      \min_{\xi} ~
         & (\xi -e_m)^\top S(\xi -e_m) \\
      \mathrm{s.t.} ~
         &
         (1-\mathbf{1}^\top \xi)g(\theta_t)
         +G\xi
         +\frac{1}{2}
         \left(\xi^\top S\xi + |\xi|^\top \mathrm{diag}(S)\right)L
         \le \mathbf{0},
   \end{aligned}
\end{equation}
where $e_m=[0,\ldots,0,1]^\top\in\mathbb R^m$ and
$|\xi|=[|\xi_1|,\ldots,|\xi_m|]^\top$.

Suppose $g(\theta_t)\le \mathbf{0}$ and
$\Vert \Delta\theta_{\mathrm{raw}} \Vert_2 \le r$. If $\xi^\star$ solves
\eqref{eq:sampled_projection}, then the induced update
$\Delta\theta^\star = D\xi^\star$ satisfies
\[
\Vert \Delta\theta^\star \Vert_2
\le
\Vert \Delta\theta_{\mathrm{raw}} \Vert_2
\le r,
\]
and is feasible for the local surrogate constraint in
\eqref{eq:safety_projection_problem_linear}. Moreover, under the
assumption $\Delta\theta_{\mathrm{raw}}=De_m$, the objective of
\eqref{eq:safety_projection_problem_linear} restricted to
$\Delta\theta\in\mathrm{span}(D)$ is equal to the objective in
\eqref{eq:sampled_projection}.
\end{proposition}

\begin{proof}
    \ifextended
   See Appendix A.
   \else
    See Appendix A in the extended version of this paper~\citep{arxiv}. 
    \fi
\end{proof}

Since $S = D^\top D \succeq 0$ and $\mathrm{diag}(S) \ge 0$ elementwise, the objective and each quadratic constraint in \eqref{eq:sampled_projection} is convex in $\xi$.
The problem can thus be viewed as a Euclidean projection onto a convex set. 


Note that the problem has complexity that depends on $m$ and $n_g$ but is independent of the model dimension $d$, and it can be solved efficiently when $m$ and $n_g$ are of moderate size.

\begin{cor}
   Let the constraints $g$ correspond to~\eqref{eq:input-feasibility}-\eqref{eq:lyap-margin-preservation}. $x_e$ is locally asymptotically stable for the closed-loop system $x_{k+1} = f(x_k, \pi_{\theta_i}(x_k))$, and $\Omega$ is forward invariant under $\pi_{\theta_i}$, $\forall i \ge 0$.  
\end{cor}

\subsection{Monotonic Improvement}
In this section, we show that the projection in our approach leads to monotonic improvement of the objective, and the sequence of $\mathcal{L}(\theta_t)$ converges under standard assumptions. 

\begin{proposition}\label{thm:decreasing-direction}
    At $t$, the optimal update $\Delta \theta^\star_t = D_t \xi^\star_t$ from~\eqref{eq:sampled_projection} is in a descent direction in $\mathcal{L}$ at $\theta_t$ if $\Delta \theta^\star_t \ne 0$. 
\end{proposition}

\begin{proof}
    \ifextended
    See Appendix B.
    \else
    See Appendix B in the extended version of this paper~\citep{arxiv}. 
    \fi
\end{proof}

Proposition~\ref{thm:decreasing-direction} indicates that with a standard Armijo-style backtracking line search, which we describe below, the algorithm provides monotonic improvement in the objective without constraint violation. 

\begin{lemma}[Feasibility of scaled updates]\label{prop:backtracking-feasibility}
    If $g(\theta_t) \leq 0$, and $\bar{\xi}$ is a feasible solution to~\eqref{eq:sampled_projection}, then $\forall \rho \in [0, 1]$, $\rho \bar{\xi}$ is also feasible. 
\end{lemma}
Lemma~\ref{prop:backtracking-feasibility} holds because the feasible set of~\eqref{eq:sampled_projection} is convex, and $g(\theta_t) \leq 0$ indicates $\xi = 0$ is a feasible solution. 

\begin{assumption}\label{asm:l-smoothness-and-boundedness-of-objective}
    Assume the objective $\mathcal{L}$ is $C^1$, $L_{\mathcal{L}}$-smooth and bounded from below. Formally, 
    \begin{equation}
        \begin{aligned}
            \Vert \nabla \mathcal{L}(\theta_1) - \nabla \mathcal{L}(\theta_2) \Vert &\leq L_{\mathcal{L}} \Vert \theta_1 - \theta_2 \Vert,~\forall \theta_1, \theta_2, \\
            \inf_\theta \mathcal{L}(\theta) &> -\infty. 
        \end{aligned}
    \end{equation}
\end{assumption}

\begin{theorem}[Monotonic improvement]\label{prop:backtracking-line-search}
Since $\Delta \theta^\star_t$ is a descent direction if $\Delta \theta^\star_t \ne 0$, as shown in Proposition~\ref{thm:decreasing-direction}, and $\mathcal{L}$ has a Lipschitz continuous gradient, an Armijo backtracking line search finds $\tau \in (0, 1]$, such that
    \begin{equation}
        \mathcal{L}(\theta_t + \tau \Delta \theta^\star_t) \leq \mathcal{L}(\theta_t) - \sigma \tau \eta^{-1} \Vert \Delta \theta^\star_t \Vert^2, \exists \sigma \in (0, 1). 
    \end{equation}
    Note that Lemma~\ref{prop:backtracking-feasibility} guarantees the feasibility of the backtracked solution $\theta_{t+1} = \theta_t + \tau \Delta \theta^\star_t$.  
\end{theorem}



\subsection{Algorithm}
In this section, we present a practical implementation of our method in Algorithm~\ref{alg:S-CPO}.
At iteration $t$, we maintain recent parameters $\theta_{t-m:t}$ and safety estimates $g_{t-m:t}$ to construct $D$ and $G$.

For safety evaluation, we use a curated validation set of safety-critical scenarios as a tractable surrogate for exhaustive verification. Designing such validation sets optimally is left for future work.

\begin{algorithm}[htbp]
\caption{Constrained Policy Optimization with Sampling-Based Weight-Space Projection (SCPO)}
\label{alg:S-CPO}
\begin{algorithmic}[1]
    \Require Feasible initialization $\theta_0$ such that $g(\theta_0)\le 0$;
    maximum sample bank size $m$; gradient step size $\eta>0$;
    curvature inflation factor $\kappa>1$
    \State $D \gets$ deque(maxlen=$m$); \quad $G \gets$ deque(maxlen=$m$)
    \State Evaluate $g(\theta_0)$
    \State Append $\Delta\theta_0\gets \mathbf 0$ to $D$
    \State Append $g_0\gets g(\theta_0)$ to $G$
    \For{$t=0,1,\dots$}
        \State Evaluate $\mathcal L(\theta_t)$ and $g(\theta_t)$
        \State $\theta_{t+\frac12}\gets \theta_t-\eta\nabla_\theta\mathcal L(\theta_t)$
        \State Evaluate $g(\theta_{t+\frac12})$
        \State Append $\Delta\theta_{t+\frac12}\gets \theta_{t+\frac12}-\theta_t$ to $D$
        \State Append $g_{t+\frac12}\gets g(\theta_{t+\frac12})$ to $G$
        \State Estimate $L$ from $D$ and $G$
        \Repeat
            \State $\xi_t^\star\gets$ solution of~\eqref{eq:sampled_projection} using $(D,G,L,g(\theta_t))$
            \State $\tau_t\gets$ step size from backtracking line search
            \State $\theta_{\mathrm{cand}}\gets \theta_t+\tau_t D \xi_t^\star$
            \State Evaluate $g(\theta_{\mathrm{cand}})$
            \If{$g(\theta_{\mathrm{cand}})\not\le 0$}
                \State $L\gets \kappa L$ \Comment{Increase conservativeness}
            \EndIf
        \Until{$g(\theta_{\mathrm{cand}})\le 0$}
        \State $\theta_{t+1}\gets \theta_{\mathrm{cand}}$
        \State $D\gets D-\tau_t D\xi_t^\star\mathbf 1^\top$
        \Comment{Recenter samples around $\theta_{t+1}$}
    \EndFor
\end{algorithmic}
\end{algorithm}

\ifextended
\subsection{Connections to Prior Work}

The proposed update rule can be interpreted as a projected gradient step in parameter space. 
In this sense, our method is structurally related to projected gradient descent and to constrained policy update methods such as CPO~\citep{achiamConstrainedPolicyOptimization2017}. We highlight here the main conceptual differences.

First, the projection in CPO involves solving a constrained natural-gradient step on linearized constraint cost and surrogate objective with a KL trust region in policy space. 
Our method instead defines a local feasible set directly in weight space from sampled trajectories and safety metrics, and uses a quadratic regularization in weight space instead of a fixed trust region. 

Second, we optimize over the low-dimensional coefficient vector $\xi$, which yields a lightweight QCQP in $\mathbb{R}^{m}$. 
By contrast, CPO's constrained step is defined in the full parameter space $\mathbb{R}^d$ and is typically approximated with matrix-free solvers in the Fisher metric.

Third, CPO and related CMDP methods impose constraints through explicit constraint cost functions. 
In our framework, safety is enforced through trajectory-based metrics $g(\theta)$ computed on a finite set of rollouts, without requiring a separate CMDP-style constraint cost. 
\fi


\section{Experiments}\label{sec:experiments}
We evaluate weight-space safety projection on both a toy regression problem and a constrained control task. The experiments are designed to answer two main questions:
\begin{itemize}
    \item \textbf{Robustness to harmful supervision.} 
    Can the proposed projection reject harmful updates while still making meaningful progress on the learning objective?
    
    \item \textbf{Safety preservation during learning.} 
    Does the learned policy remain safe throughout training? In the control setting, does the policy remain stabilizing at every training epoch?
\end{itemize}



\subsection{Constrained Regression}
We first consider a one-dimensional constrained regression problem to illustrate how the proposed method filters harmful updates.
\ifextended

In this regression example only, $f$ denotes the scalar reference function, not the system dynamics.
\fi
The constrained learning objective is
\begin{equation}\label{eq:constrained-regression-problem}
    \begin{aligned}
        \min_{\theta} \quad & 
        \mathbb{E}_{x \sim \mathcal{N}(0,1)}\!\left[\,
        \bigl(f(x)-\pi_\theta(x)\bigr)^2 \right] \\
        \text{s.t.}\quad & 
        \left|\pi_\theta(x)\right| \le 1.4,\quad \forall x\in[-3,3],
    \end{aligned}
\end{equation}
where the reference function is
$f(x)=\sin(x)+\sin(3x)+\sin(7x)$.
Note that $f$ violates the constraint on part of $[-3,3]$, so na\"{i}vely regressing onto $f(x)$ leads to constraint violation.

Let the safe policy be $\pi_{\text{safe}}(x)\equiv 0$, which satisfies the constraint. 
We enforce safety on a uniform grid $\mathbf{v}\in\mathbb{R}^{64}$ over $[-3,3]$.
The surrogate constraint violation vector is
\begin{equation}
    g(\theta) 
    = \vert \pi_\theta(\mathbf{v}) \vert - 1.4 \times \mathbf{1},
\end{equation}
where $\pi_\theta(\mathbf{v})$ denotes elementwise evaluation on the grid. 
For reporting, we summarize violation by $\|g(\theta)\|_1/64$.

At each epoch, we draw $N=64$ i.i.d.\ samples from $\mathcal{N}(0,1)$ and minimize the empirical loss
\begin{equation}
    \mathcal{L}(\theta) 
    = \frac{1}{N}\sum_{i=1}^{N} \bigl(f(x_i)-\pi_\theta(x_i)\bigr)^2.
\end{equation}
Fig.~\ref{fig:regression} shows the learned curve after training. 
Starting from a feasible initialization, projection prevents unsafe drift: the policy fits $f$ where it is safe and smoothly saturates at the constraint boundary with no violation. 

\begin{figure}[htbp]
    \centering
    \includegraphics[width=0.9\linewidth]{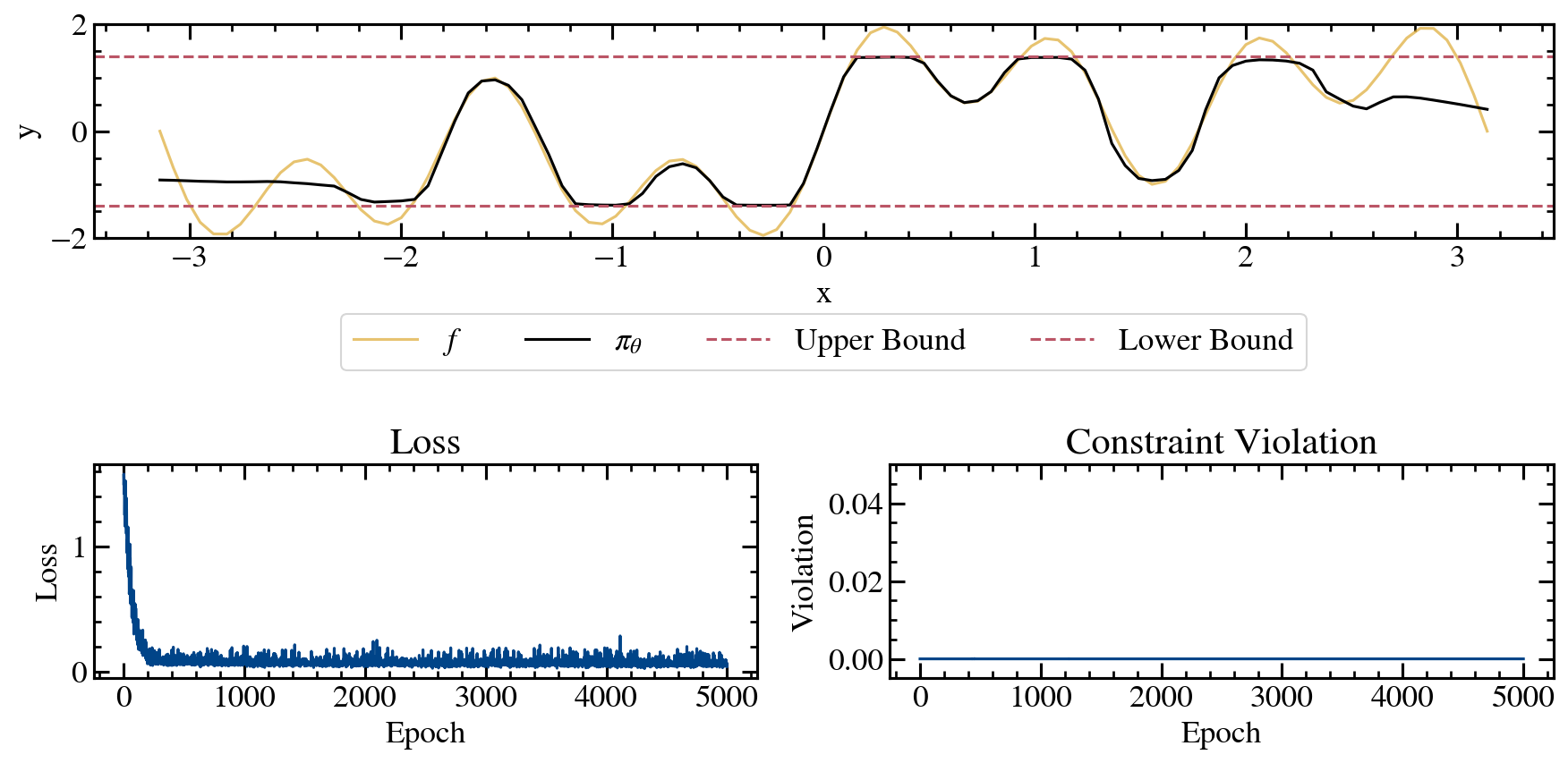}
    \caption{Constrained regression with proposed projection. }
    \label{fig:regression}
\end{figure}


To compare our approach with a standard baseline, we compare against a soft-constraint baseline. Specifically, we replace the hard projection step with an unconstrained optimization problem augmented by a soft penalty: 
\begin{equation}
    \min_{\theta}\; 
    \mathbb{E}_{x \sim \mathcal{N}(0,1)}\bigl[(f(x)-\pi_\theta(x))^2\bigr]
    + \lambda_p \cdot \mathbf{1}^\top (\mathrm{ReLU}(g(\theta))),
\end{equation}
where $\lambda_p > 0$ is a constant. 

\begin{remark}
Note that this comparison is not fully fair: this baseline assumes access to first-order information of $g$, either through a known differentiable constraint or a gradient estimate, as in CPO or PPO-Lagrangian. This is outside our setting, where only pointwise evaluations of $g(\theta)$ are available. The comparison is included to show that our method can achieve comparable performance without this advantage.
\end{remark}

As shown in Fig.~\ref{fig:regression-with-soft-constraint}, even when such a differentiable soft penalty is artificially made available, it still fails to guarantee strict constraint satisfaction and introduces an additional hyperparameter $\lambda_p$ that must be tuned.  

\begin{figure}[htbp]
    \centering
    \includegraphics[width=0.9\linewidth]{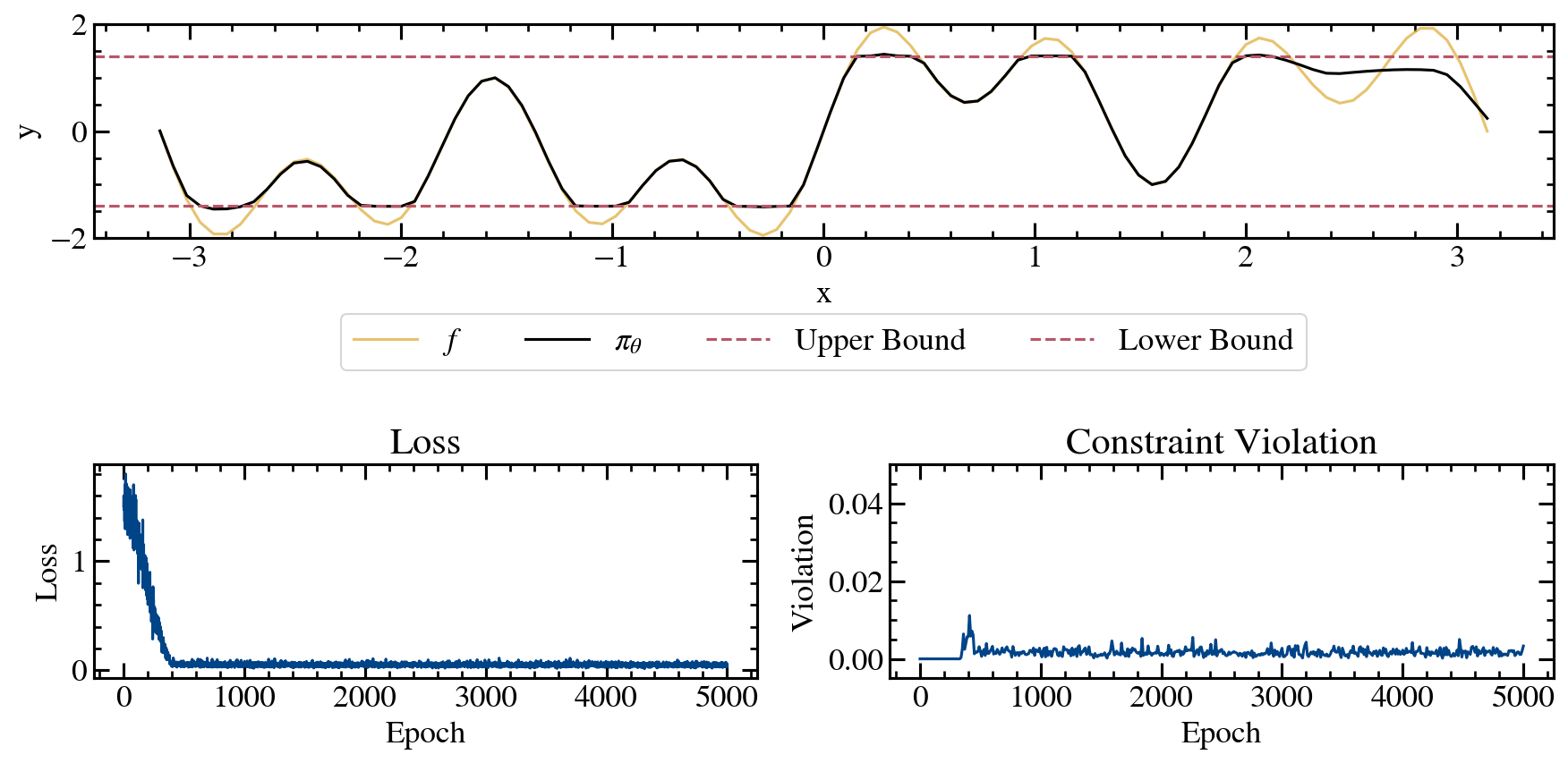}
    \caption{Regression with soft constraint ($\lambda_p = 1$), without proposed projection.}
    \label{fig:regression-with-soft-constraint}
\end{figure}

\subsection{Double Integrator Imitation Learning}

We next evaluate SCPO on a constrained control task where the learning signal is intentionally harmful. The purpose of this experiment is to test whether SCPO can improve the imitation objective while preserving the stabilizing behavior of a known safe controller.

Consider the discrete-time double integrator
\begin{equation}
\label{eq:double_integrator_example}
\begin{aligned}
    x_{k+1} &=
    \begin{bmatrix}1 & \Delta t \\ 0 & 1\end{bmatrix}x_k
    +
    \begin{bmatrix}\frac{1}{2}\Delta t^2 \\ \Delta t\end{bmatrix}u_k, \\
    \begin{bmatrix}-15 \\ -15\end{bmatrix}
    &\le x_k \le
    \begin{bmatrix}15 \\ 15\end{bmatrix}, \qquad
    -1 \le u_k \le 1,
\end{aligned}
\end{equation}
with $\Delta t=0.1$. The input constraint is enforced by clipping all control inputs to $[-1,1]$. The target set is a small neighborhood of the equilibrium at $x_e = \mathbf{0}$,
\begin{equation}
    \mathcal{X}_f=\{x\mid \|x\|_2\le 0.01\}.
\end{equation}

The safe controller is a clipped LQR feedback controller,
\begin{equation}
    \pi_{\mathrm{safe}}(x)=\mathrm{clip}(-Kx,-1,1),
\end{equation}
where $K$ is obtained from the DARE with $Q=I_2$ and $R=I_1$. The expert policy is intentionally chosen to be harmful:
\begin{equation}
    \pi_\beta(x)=\mathrm{clip}\bigl(-2(x_1+x_2)+\delta(x),-1,1\bigr),
\end{equation}
where $\delta(x)$ is generated by a fixed seeded random function. Thus, each state is assigned a fixed perturbation, while the perturbations vary across the state space. This produces an aggressive and noisy expert whose behavior reduces the stabilizing region.

We train $\pi_\theta$ using the on-policy imitation objective
\begin{equation}
    \mathcal{L}(\theta)
    =
    \mathbb{E}_{x\sim\mathcal{D}_\theta}
    \bigl[\|\pi_\theta(x)-\pi_\beta(x)\|^2\bigr],
\end{equation}
where $\mathcal{D}_\theta$ is the state distribution induced by rolling out the current policy.

To instantiate the Lyapunov-margin constraint, we use the value function induced by the safe controller,
\begin{equation}
   V(x_0)
   =
   \sum_{k=0}^{\infty}
   \left[
   \|x_k\|_2^2+\|u_k\|_2^2
   \right],
   \qquad
   u_k=\pi_{\mathrm{safe}}(x_k),
\end{equation}
with the state evolving according to~\eqref{eq:double_integrator_example}. Since $\pi_{\mathrm{safe}}$ stabilizes the origin, this value function is a valid Lyapunov function on its region of attraction.

Fig.~\ref{fig:double-integrator-trajectories} compares closed-loop trajectories from a representative initial condition under $\pi_{\mathrm{safe}}$, $\pi_\theta$, and $\pi_\beta$. The learned policy moves toward the expert behavior while avoiding the destabilizing behavior of the expert, illustrating that the projection rejects harmful updates during training. Fig.~\ref{fig:double-integrator-o-inf} compares the estimated regions from which each controller reaches $\mathcal{X}_f$ without constraint violation. The learned policy preserves the stabilizing region of $\pi_{\mathrm{safe}}$, whereas the malicious expert loses stabilizing capability over a substantial portion of the state space.


\begin{figure}[htbp]
    \centering
    \includegraphics[width=0.65\linewidth]{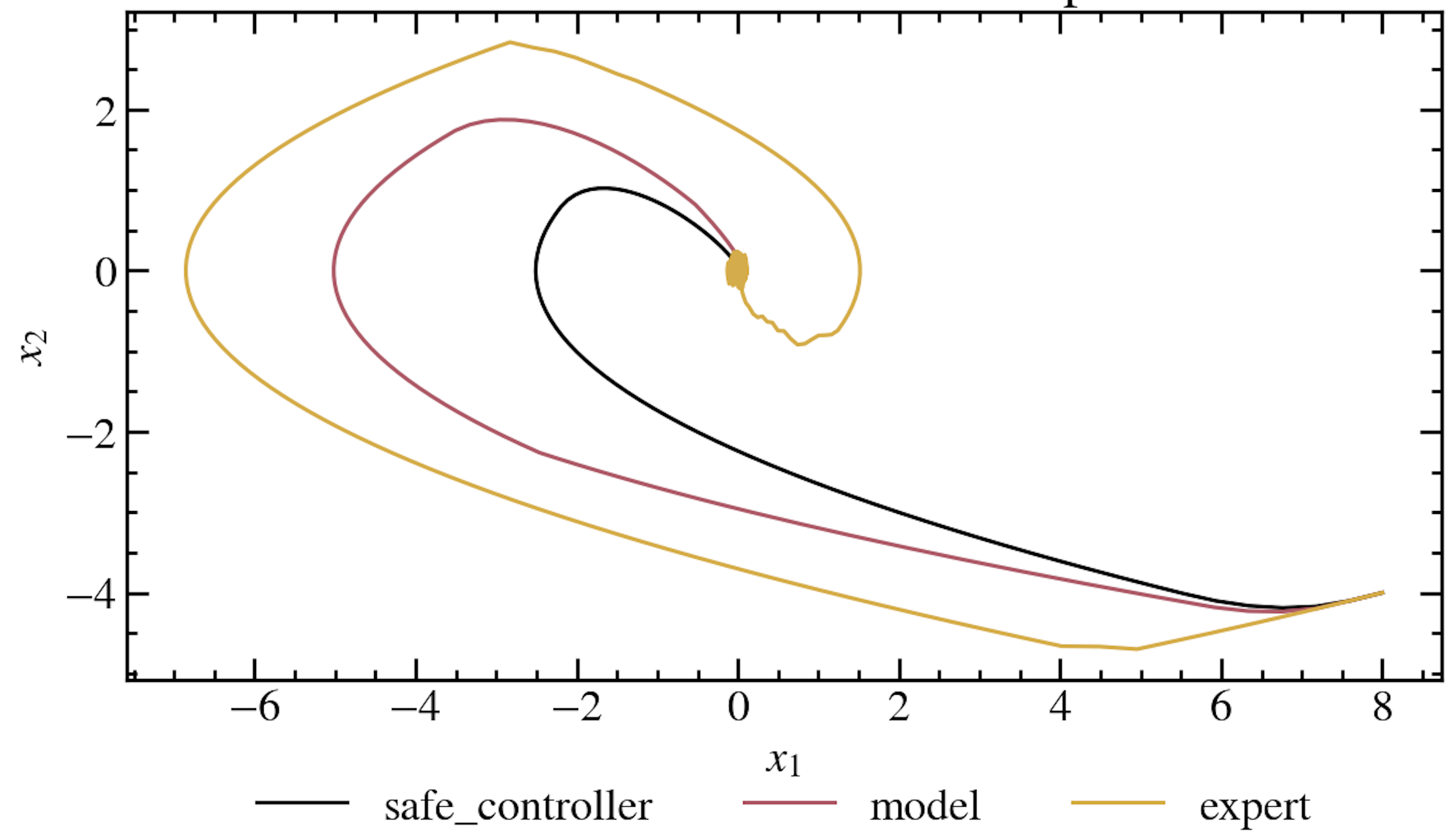}
    \caption{Closed-loop trajectories from an example initial state for the double integrator under $\pisafe$, $\pi_\beta$, and the policy $\pi_\theta$ at epoch 16. }
    \label{fig:double-integrator-trajectories}
\end{figure}


\begin{figure}[htbp]
    \centering
    \includegraphics[width=0.65\linewidth]{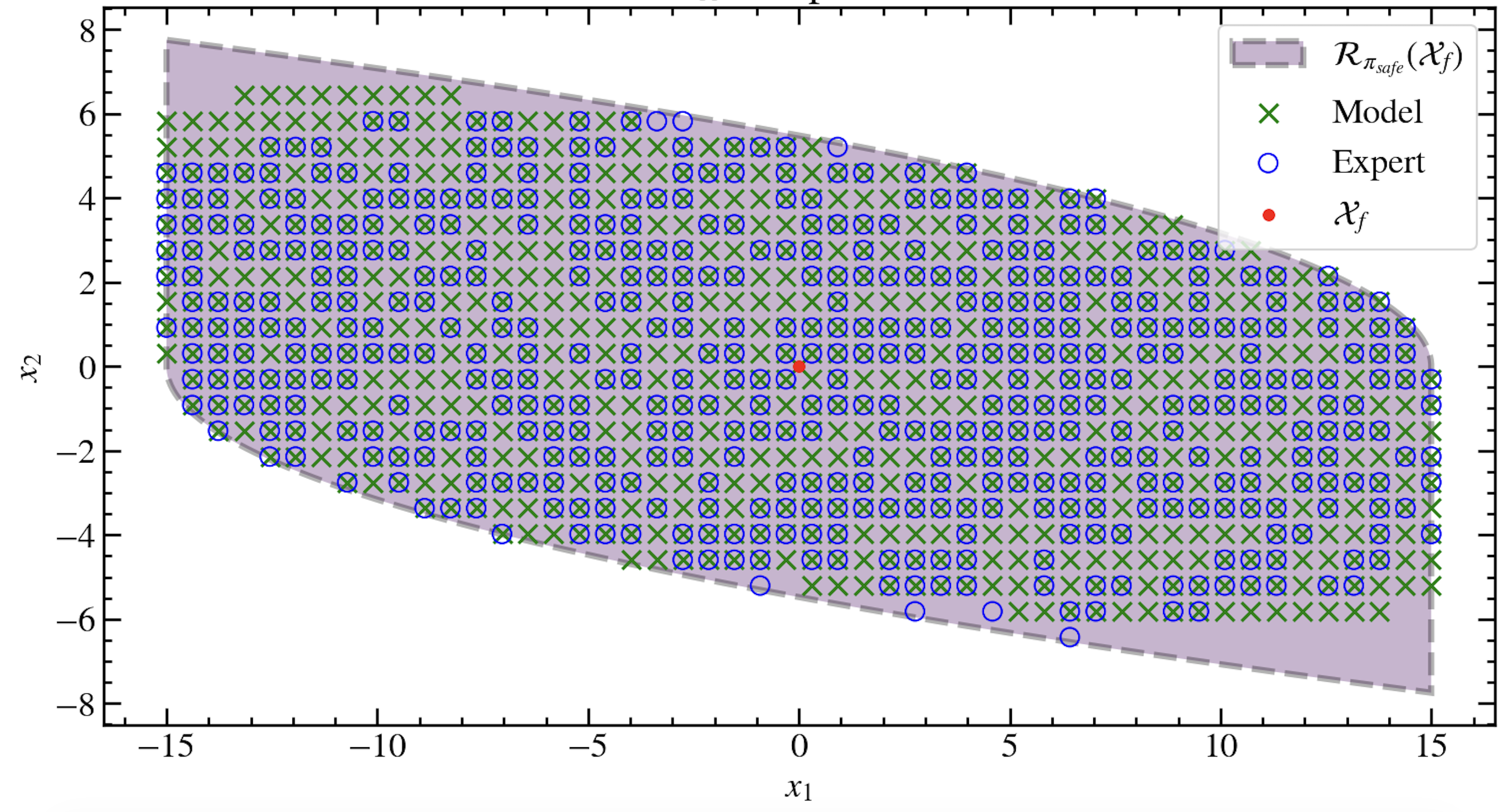}
    \caption{Comparison between the states from which the closed-loop system can reach $\mathcal{X}_f$ under the trained policy $\pi_{\theta_N}$ and the malicious expert $\pi_\beta$. The shaded area is the set of all states from which there exists a stabilizing control sequence.}
    \label{fig:double-integrator-o-inf}
\end{figure}

\ifextended
\section{Conclusion}
We introduced a weight-space projection with guarantees for safe-by-induction feasibility and monotonic improvement. 
A primary limitation is the method's conservativeness as the projection step does not generate new direction to improve. We aim to address this by integrating~\cite{caoSimpleApproachConstraintAware2025} to modify the objective for improved safe exploration.

\begin{ack}
    The authors would like to thank Xiao (Sam) Huang, Hansung Kim, Fionna Koop, and Yilun Ma for the valuable discussions and feedback during the preparation of this manuscript. 
\end{ack}

\section*{DECLARATION OF GENERATIVE AI AND AI-ASSISTED TECHNOLOGIES IN THE WRITING PROCESS}
During the preparation of this work the authors used ChatGPT in order to proofread the manuscript. After using this tool/service, the authors reviewed and edited the content as needed and take full responsibility for the content of the publication.
\fi

\bibliography{ifacconf}             

\ifextended
\appendix
\section{Proof of Proposition~\ref{prop:reformulation}}

\begin{proof}
Under the subspace parameterization $\Delta\theta = D\xi$,
\eqref{eq:safety_projection_problem_linear} becomes
\begin{equation}
    \label{eq:projection_subspace}
    \begin{aligned}
        \min_{\xi} \quad
            & \Vert D\xi - \Delta \theta_{\mathrm{raw}} \Vert_2^2 \\
        \mathrm{s.t.} \quad
            &
            g(\theta_t) + J_g(\theta_t)D\xi
            + \frac{1}{2}\Vert D\xi\Vert_2^2 L
            \le \mathbf{0}.
    \end{aligned}
\end{equation}
Since $\Delta\theta_{\mathrm{raw}}=\Delta\theta^{(m)}=De_m$, the objective is
\begin{equation}
   \Vert D\xi-\Delta\theta_{\mathrm{raw}}\Vert_2^2
   =
   \Vert D(\xi-e_m)\Vert_2^2
   =
   (\xi-e_m)^\top S(\xi-e_m).
\end{equation}

We next upper-bound the unknown Jacobian term using only sampled
evaluations of $g$. Fix a constraint component $g_j$ and a sampled update
$\Delta\theta^{(i)}$. By Assumption~\ref{ass:l_smooth},
\begin{equation}
   \begin{aligned}
      g_j(\theta_t+\Delta\theta^{(i)})
      &=
      g_j(\theta_t)
      +
      \nabla_\theta g_j(\theta_t)^\top \Delta\theta^{(i)}
      +
      \varepsilon_{ij}, \\
      |\varepsilon_{ij}|
      &\le
      \frac{\ell_j}{2}\|\Delta\theta^{(i)}\|_2^2 .
   \end{aligned}
\end{equation}
   Therefore, for any coefficient $\xi_i$,
\begin{equation}
   \begin{aligned}
      \xi_i\nabla_\theta g_j(\theta_t)^\top\Delta\theta^{(i)}
      & \le
      \xi_i\left(
      g_j(\theta_t+\Delta\theta^{(i)})-g_j(\theta_t)
      \right) \\
      & \qquad +
      |\xi_i|\frac{\ell_j}{2}\|\Delta\theta^{(i)}\|_2^2 .
   \end{aligned}
\end{equation}
Summing over $i=1,\ldots,m$ gives
\begin{equation}
    \label{eq:jacobian_bound_compact}
    \begin{aligned}
    \relax[J_g(\theta_t)D\xi]_j
    &=
    \sum_{i=1}^m \xi_i
    \nabla_\theta g_j(\theta_t)^\top\Delta\theta^{(i)} \\
    &\le
    [G\xi]_j
    -
    (\mathbf{1}^\top \xi)g_j(\theta_t)
    +
    \frac{\ell_j}{2}|\xi|^\top\mathrm{diag}(S).
    \end{aligned}
\end{equation}
Moreover,
\[
\|D\xi\|_2^2 = \xi^\top S\xi.
\]
Substituting~\eqref{eq:jacobian_bound_compact} into the conservative
constraint in~\eqref{eq:projection_subspace} gives
\[
(1-\mathbf{1}^\top \xi)g_j(\theta_t)
+
[G\xi]_j
+
\frac{\ell_j}{2}
\left(
\xi^\top S\xi + |\xi|^\top\mathrm{diag}(S)
\right)
\le 0.
\]
Stacking these inequalities for $j=1,\ldots,n_g$ gives the constraint in
\eqref{eq:sampled_projection}. Therefore, any feasible $c$ for
\eqref{eq:sampled_projection} satisfies the local surrogate safety
constraint.

It remains to show that the projected update remains inside the trust
region. Let

\begin{equation}
   \begin{aligned}
   \Xi
   =
   \bigl\{
      \xi\in\mathbb R^m:~&
      (1-\mathbf{1}^\top \xi)g(\theta_t)
      +G\xi \\
      &~~ +\frac{1}{2}
      \left(\xi^\top S\xi+|\xi|^\top\mathrm{diag}(S)\right)L
      \le \mathbf 0
      \bigr\}.
   \end{aligned}
\end{equation}
Since $\ell_j\ge 0$ for all $j$, $\Xi$ is convex. Moreover,
$\xi=0$ belongs to $\Xi$ because
\begin{equation}
    g(\theta_t)\le \mathbf 0.
\end{equation}
Thus, the corresponding feasible set in update space,
\begin{equation}
    D\Xi = \{D\xi:\xi\in\Xi\},
\end{equation}
is also convex and contains the zero update.

The solution $\Delta\theta^\star=D\xi^\star$ is the Euclidean projection of
$\Delta\theta_{\mathrm{raw}}$ onto $D\Xi$. Since $0\in D\Xi$,
the projection cannot increase the norm:
\begin{equation}
    \|\Delta\theta^\star\|
    \le
    \|\Delta\theta_{\mathrm{raw}}\|.
\end{equation}
Because the raw update is scaled to satisfy
$\|\Delta\theta_{\mathrm{raw}}\|\le r$, we obtain
\begin{equation}
    \|\Delta\theta^\star\|\le r.
\end{equation}
Hence the projected update remains in the trust region, and the local
smoothness bound applies.
\end{proof}

\section{Proof of Proposition~\ref{thm:decreasing-direction}}

\begin{proof}
    Let $\mathcal{C}_t$ be the feasible set of~\eqref{eq:sampled_projection} at $t$.
    If $e_m \in \mathcal{C}_t$, then $\xi^\star_t = e_m$, and $\Delta \theta_t^\star = \Delta \theta_{\text{raw}}$ is a descent direction. 
    
    Let $\Theta_t = \{\Delta \theta_t = D_t \xi \mid \xi \in \mathcal{C}_t \}$. 
    If $e_m \notin \mathcal{C}_t$, note that $\Theta_t$ is an image of the convex set $\mathcal{C}_t$ under a linear transformation, and therefore $\Theta_t$ is convex in $\Delta \theta_t$, and~\eqref{eq:sampled_projection} is a Euclidean projection onto a convex set. 
    By hyperplane separation theorem, \begin{equation}\label{eq:Euclidean-projection-property}
        \langle \Delta \theta_{\text{raw}} -\Delta \theta^\star_t, x - \Delta \theta^\star_t \rangle \leq 0, \forall x \in \Theta_t.
    \end{equation}
    Note that $\mathbf{0} \in \mathcal{C}_t$ following the assumption that $g(\theta_t) \leq 0$. 
    Let $x = \mathbf{0}$, 
    \begin{equation}
        \langle \Delta \theta_{\text{raw}}, \Delta \theta^\star_t \rangle \ge \Vert \Delta \theta^\star_t \Vert_2^2. 
    \end{equation}
    Recall $\Delta \theta_{\text{raw}} = -\eta \nabla_\theta \mathcal{L}(\theta_t)$, where $\eta > 0$. Therefore, 
    \begin{equation}\label{eq:descent-direction-proof}
        \begin{aligned}
            -\nabla_\theta \mathcal{L}(\theta_t)^\top \Delta \theta^\star_t & \geq \frac{1}{\eta} \Vert \Delta \theta^\star_t \Vert_2^2 > 0.
        \end{aligned}
    \end{equation} 
\end{proof}
\fi

\end{document}